\documentclass[11pt]{article}
\usepackage{algorithmic,amsmath,amssymb,amsthm, graphicx, url, algorithm2e}
\usepackage[margin=1.25in]{geometry}

\title{Algorithmic Connections Between Active Learning and Stochastic Convex Optimization}

\author{
Aaditya Ramdas \\
Machine Learning Department\\
Carnegie Mellon University\\
\texttt{aramdas@cs.cmu.edu}
\and 
Aarti Singh \\
Machine Learning Department\\
Carnegie Mellon University\\
\texttt{aarti@cs.cmu.edu}
}

\newtheorem{lemma}{Lemma}
\newtheorem{theorem}{Theorem}

\newcommand{\E}{\mathbb{E}}
\newcommand{\R}{\mathbb{R}}
\newcommand{\td}{{\tilde{\delta}}}
\newcommand{\bT} {\mathrm{\Theta}}
\newcommand{\bO} {\mathrm{O}}
\newcommand{\tO}{\mathrm{{\tilde{O}}}}
\newcommand{\tT}{\mathrm{{\tilde{\Theta}}}}
\newcommand{\Risk}{\mathcal{R}}
\newcommand{\hg}{{\hat{g}}}
\newcommand{\hs}{{\hat{s}}}

\begin{document}

\maketitle

\begin{abstract}
Interesting theoretical associations have been established by recent papers between the fields of active learning and stochastic convex optimization due to the common role of feedback in sequential querying mechanisms. In this paper, we continue this thread in two parts by exploiting these relations for the first time to yield novel algorithms in both fields, further motivating the study of their intersection. First, inspired by a recent optimization algorithm that was adaptive to unknown uniform convexity parameters, we present a new active learning algorithm for one-dimensional thresholds that can yield minimax rates by adapting to unknown noise parameters. Next, we show that one can perform $d$-dimensional stochastic minimization of smooth uniformly convex functions when only granted oracle access to noisy gradient signs along any coordinate instead of real-valued gradients, by using a simple randomized coordinate descent procedure where each line search can be solved by $1$-dimensional active learning, provably achieving the same error convergence rate as having the entire real-valued gradient. Combining these two parts yields an algorithm that solves stochastic convex optimization of uniformly convex and smooth functions using only noisy gradient signs by repeatedly performing active learning, achieves optimal rates and is adaptive to all unknown convexity and smoothness parameters.
\end{abstract}

\section{Introduction}

The two fields of convex optimization and active learning seem to have evolved quite independently of each other. Recently, \cite{RR09} pointed out their relatedness due to the inherent sequential nature of both fields and the complex role of feedback in taking future actions. Following that, \cite{RS13} made the connections more explicit by tying together the exponent used in noise conditions in active learning and the exponent used in uniform convexity (UC) in optimization. They used this to establish lower bounds (and tight upper bounds) in stochastic optimization of UC functions based on proof techniques from active learning. However, it was unclear if there were concrete algorithmic ideas in common between the fields.

Here, we provide a positive answer by exploiting the aforementioned connections to form new and interesting algorithms that clearly demonstrate that the complexity of $d$-dimensional stochastic optimization is precisely the complexity of $1$-dimensional active learning. Inspired by an optimization algorithm that was adaptive to unknown uniform convexity parameters, we design an interesting one-dimensional active learner that is also adaptive to unknown noise parameters. This algorithm is simpler than the adaptive active learning algorithm proposed recently in \cite{H11} which handles the pool based active learning setting. 

Given access to this active learner as a subroutine for line search, we show that a simple randomized coordinate descent procedure can minimize uniformly convex functions with a much simpler stochastic oracle that returns only a Bernoulli random variable representing a noisy sign of the gradient in a single coordinate direction, rather than a full-dimensional real-valued gradient vector. The resulting algorithm is adaptive to all unknown UC and smoothness parameters and achieve minimax optimal convergence rates.

We spend the first two sections describing the problem setup and preliminary insights, before describing our algorithms in sections 3 and 4.



\subsection{Setup of First-Order Stochastic Convex Optimization}

First-order stochastic convex optimization is the task of approximately minimizing a convex function over a convex set, given oracle access to unbiased estimates of the function and gradient at any point, using as few queries as possible (\cite{NY83}).

We will assume that we are given an arbitrary set $S\subset \R^d$ of known diameter bound $R = \max_{x,y\in S} \|x-y\|$. A convex function $f$ with $x^* = \arg \min_{x \in S} f(x)$ is said to be $k$-uniformly convex if, for some $\lambda > 0, k \geq 2$, we have for all $x,y \in S$
$$f(y) \geq f(x) + \nabla f(x)^\top (y-x) + \frac{\lambda}{2} \|x-y\|^k$$
(strong convexity arises when $k=2$). $f$ is $L$-Lipschitz for some $L>0$ if $\|\nabla f(x)\|_* \leq L$ (where $\|.\|_*$ is the dual norm of $\|.\|$); equivalently for all $x,y \in S$
\begin{equation*}
|f(x) - f(y)| \leq L \|x-y\|
\end{equation*}
A differentiable $f$ is $H$-strongly smooth (or has a $H$-Lipschitz gradient) for some $H>\lambda$ if for all $x,y \in S$, we have $\|\nabla f(x) - \nabla f(y)\|_* \leq H \|x-y\|$, or equivalently
$$f(y) \leq f(x) + \nabla f(x)^\top (y-x) + \frac{H}{2} \|x-y\|^2$$
In this paper we shall always assume $\|.\| = \|.\|_*=\|.\|_2$ and deal with strongly smooth and uniformly convex functions with parameters $\lambda > 0, k \geq 2$, $L,H>0$.\\
A stochastic first order oracle is a function that accepts $x \in S$, and returns 
$$\Big(\hat{f}(x),\hg(x) \Big) \in \R^{d+1} \mbox{ where } \E \big[\hat{f}(x) \big] = f(x), \E\big[\hat{g}(x)\big]= \nabla f(x)$$
(these unbiased estimates also have bounded variance) and the expectation is over any internal randomness of the oracle. \\
An optimization algorithm is a method that sequentially queries an oracle at points in $S$ and returns $\hat{x}_T$ as an estimate of the optimum of $f$ after $T$ queries (or alternatively tries to achieve an error of $\epsilon$) and their performance can be measured by either function error $f(\hat{x}_T) - f(x^*)$ or point error $\|\hat{x}_T - x^*\|$.\\

\subsection{Stochastic Gradient-Sign Oracles} \label{sgso}
Define a stochastic sign oracle to be a function of $x \in S, j \in \{1...d\}$, that returns 
$$\hs_j(x) \in \{+,-\} \mbox{ where}\footnote{$f=\bT (g)$ means $f=\mathrm{\Omega}(g)$ and $f=\bO (g)$ (rate of growth)} \  \big |\eta(x) - 0.5 \big | = \bT \Big( [\nabla f(x)]_j \Big) \mbox{ and } \eta(x) = \Pr \big ( \hs_j(x) = +  | x \big )$$
where $\hs_j(x)$ is a noisy sign$\big( [\nabla f(x)]_j \big)$ and $[\nabla f(x)]_j$ is the $j$-th coordinate of $\nabla f$, and the probability is over any internal randomness of the oracle. This behavior of $\eta(x)$ actually needs to hold only when $\big |[\nabla f(x)]_j \big|$ is small.

In this paper, we consider coordinate descent algorithms that are motivated by applications where computing the overall gradient, or even a function value, can be expensive due to high dimensionality or huge amounts of data, but computing the gradient in any one coordinate can be cheap. \cite{N10} mentions the example of  $\min_x \frac1{2}\|Ax-b\|^2 + \frac1{2}\|x\|^2$ for some $n \times d$ matrix $A$ (or any other regularization that decomposes over dimensions). Computing the gradient $A^\top (Ax-b) + x$ is expensive, because of the matrix-vector multiply. However, its $j$-th coordinate is $2A^{j\top} (Ax-b) + x_j$ and requires an expense of only $n$ if the residual vector $Ax-b$ is kept track of (this is easy to do, since on a single coordinate update of $x$, the residual change is proportional to $A^j$, an additional expense of $n$). 

A sign oracle is weaker than a first order oracle, and can actually be obtained by returning the sign of the first order oracle's noisy gradient if the mass of the noise distribution grows linearly around its zero mean (argued in next section). At the optimum along coordinate $j$, the oracle returns a $\pm 1$ with equal probability, and otherwise returns the correct sign with a probability proportional to the value of the directional derivative at that point (this is reflective of the fact that the larger the derivative's absolute value, the easier it would be for the oracle to approximate its sign, hence the smaller the probability of error). It is not unreasonable that there may be other circumstances where even calculating the (real value) gradient in the $i$-th direction could be expensive, but estimating its sign could be a much easier task as it only requires estimating whether function values are expected to increase or decrease along a coordinate (in a similar spirit of function comparison oracles \cite{JNR12}, but with slightly more power). 

We will also see that the rates for optimization crucially depend on whether the gradient noise is sign-preserving or not. For instance, with rounding errors or storing floats with small precision, one can get deterministic rates as if we had the exact gradient since the rounding or lower precision doesn't  flip signs.





\subsection{Setup of Active Threshold Learning}

The problem of one-dimensional threshold estimation assumes you have an interval of length $R$, say $[0,R]$. Given a point $x$, it has a label $y \in \{+,-\}$ that is drawn from an unknown conditional distribution $\eta(x) = \Pr \big( Y=+|X=x\big)$ and the threshold $t$ is the unique point where $\eta(x) = 1/2$, with it being larger than half on one side of $t$ and smaller than half on the other (hence it is more likely to draw a $+$ on one side of $t$ and a $-$ on the other side). 

The task of active learning of threshold classifiers allows  the learner to sequentially query $T$ (possibly dependent) points, observing labels drawn from the unknown conditional distribution after each query, with the goal of returning a guess $\hat{x}_T$ as close to $t$ as possible. In the formal study of classification (cf. \cite{T04}), it is common to study minimax rates when the regression function $\eta(x)$ satisfies Tsybakov's noise or margin condition (TNC) with exponent $k$ at the threshold $t$. Different versions of this boundary noise condition are used in regression, density or level-set estimation and lead to an improvement in minimax optimal rates (for classification, also cf. \cite{AT07}, \cite{H11}). Here, we present the version of TNC used in \cite{CN07} :
$$M |x-t|^{k - 1} \geq | \eta(x) - 1/2 | \geq \mu |x-t|^{k - 1} \mbox{ whenever}\footnote{Note that $|x-t| \leq \delta_0 := \left( \frac{\epsilon_0}{M} \right)^{\frac1{k-1}} \implies |\eta(x) - 1/2| \leq \epsilon_0 \implies |x-t| \leq \left( \frac{\epsilon_0}{\mu} \right)^{\frac1{k-1}}$} \ |\eta(x) - 1/2| \leq \epsilon_0 $$
for some constants $M>\mu>0,\epsilon_0 > 0, k \geq 1$. 

A standard measure for how well a classifier $h$ performs is given by its risk, which is simply the probability of classification error (expectation under $0-1$ loss), $\Risk(h) = \Pr \big[ h (x) \neq y \big]$. The performance of threshold learning strategies can be measured by the excess classification risk of the resultant threshold classifier at $\hat{x}_T$ compared to the Bayes optimal classifier at $t$ as given by \footnote{$a \vee b := \max(a,b) \mbox{ and } a \wedge b := \min(a,b)$}
\begin{equation} \label{risk}
\Risk (\hat{x}_T)  - \Risk (t) = \int\limits_{\hat{x}_T \wedge t}^{\hat{x}_T \vee t} | 2 \eta(x) - 1| dx
\end{equation}

In the above expression, akin to \cite{CN07}, we use a uniform marginal distribution for active learning since there is no underlying distribution over $x$. Alternatively, one can simply measure the one-dimensional point error $|\hat{x}_T - t|$ in estimation of the threshold. Minimax rates for estimation of risk and point error in active learning under TNC were provided in \cite{CN07} and are summarized in the next section.

\subsection{Summary of Contributions}

Now that we have introduced the notation used in our paper and some relevant previous work (more in the next section), we can clearly state our contributions.

\begin{itemize}

\item We generalize an idea from \cite{JN10} to present a simple epoch-based active learning algorithm with a passive learning subroutine that can optimally learn one-dimensional thresholds and is adaptive to unknown noise parameters.

\item We show that noisy gradient signs suffice for minimization of uniformly convex functions by proving that a random coordinate descent algorithm with an active learning line-search subroutine achieves minimax convergence rates.

\item Due to the connection between the relevant exponents in the two fields, we can combine the above two methods to get an algorithm that achieves minimax optimal rates and is adaptive to unknown convexity parameters.

\item As a corollary, we argue that with access to possibly noisy non-exact gradients that don't switch any signs (rounding errors or low-precision storage are sign-preserving), we can still achieve exponentially fast deterministic rates.

\end{itemize}

\section{Preliminary Insights}

\subsection{Connections Between Exponents}

Taking one point as $x^*$ in the definition of UC, we see that
$$|f(x) - f(x^*)| \geq \frac{\lambda}{2} \|x-x^*\|^k$$
Since $\|\nabla f(x)\| \|x-x^*\| \geq \nabla f(x)^\top (x-x^*) \geq f(x) - f(x^*)$ (by convexity),
$$\|\nabla f(x) - 0\| \geq \frac{\lambda}{2} \|x-x^*\|^{k-1} $$
Another relevant fact for us will be that uniformly convex functions in $d$ dimensions are uniformly convex along any one direction, or in other words, for every fixed $x \in S$ and fixed unit vector $u \in \R^d$, the univariate function of $\alpha$ defined by $f_{x,u}(\alpha) := f(x + \alpha u)$ is also UC with the same parameters\footnote{Since $f$ is UC, $f_{x,u}(\alpha) \geq f_{x,u}(0) + \alpha \nabla f_{x,u}(0) + \frac{\lambda}{2}|\alpha|^k$}. For $u = e_j$,
$$\big | [\nabla f(x)]_j - 0 \big | \geq \frac{\lambda}{2} \|x-x_{j}^*\|^{k-1}$$
where $x_{j}^* = x + \alpha_j^* e_j$ and $\alpha_{j}^* = \arg \min_{\{\alpha|x + \alpha e_j \in S\}} f(x + \alpha e_j)$. This uncanny similarity to the TNC (since $\nabla f(x^*) = 0$) was mathematically exploited in \cite{RS13} where the authors used a lower bounding proof technique for one-dimensional active threshold learning from \cite{CN07} to provide a new lower bounding proof technique for the $d$-dimensional stochastic convex optimization of UC functions. In particular, they showed that the minimax rate for $1$-dimensional active learning excess risk and the $d$-dimensional optimization function error both scaled like\footnote{we use $\tO, \tT$ to hide constants and polylogarithmic factors} $\tT \left( T^{-\frac{k}{2k-2}}\right)$, and that the point error in both settings scaled like $\tT \left( T^{-\frac{1}{2k-2}}\right)$, where $k$ is either the TNC exponent or the UC exponent, depending on the setting. The importance of this connection cannot be emphasized enough and we will see this being useful throughout this paper.\\
As mentioned earlier \cite{CN07} require a two-sided TNC condition (upper and lower growth condition to provide exact tight rate of growth) in order to prove risk upper bounds. On a similar note, for uniformly convex functions, we will assume such a Local $k$-Strong Smoothness condition around directional minima
$$\mbox{\textbf{Assumption LkSS} : \ \ \ \ for all $j \in \{1...d\}$\ \ \ } \big | [\nabla f(x)]_j - 0 \big | \leq \Lambda \|x-x_{j}^*\|^{k-1} $$
for some constant $\Lambda > \lambda/2$, so we can tightly characterize the rate of growth as
$$\big | [\nabla f(x)]_j - 0 \big | = \bT \Big( \|x-x_{j}^*\|^{k-1} \Big)$$
This condition is implied by strong smoothness or Lipschitz smooth gradients when $k=2$ (for strongly convex and strongly smooth functions), but is a slightly stronger assumption otherwise.


\subsection{The One-Dimensional Argument}

The basic argument for relating optimization to active learning was made in \cite{RS13} in the context of stochastic first order oracles when the noise distribution $\mathrm{P}(z)$ is unbiased and grows linearly around its zero mean, i.e.
$$ \int_0^\infty \mathrm{dP}(z) = \tfrac{1}{2} \ \mbox{ and } \ \int_0^t \mathrm{dP}(z) = \bT ( t ) $$
for all $0 <t < t_0$, for constants $t_0$ (similarly for $-t_0 < t < 0$). This is satisfied for gaussian, uniform and many other distributions. We reproduce the argument for clarity and then sketch it for stochastic signed oracles as well.

For any $x \in S$, it is clear that $f_{x,j}(\alpha) := f(x+\alpha e_j)$ is convex; its gradient $\nabla f_{x,j}(\alpha) := [\nabla f(x + \alpha e_j)]_j$ is an increasing function of $\alpha$ that switches signs at $\alpha^*_j := \arg\min_{\{\alpha | x+ \alpha e_j \in S\}} f_{x,j}(\alpha)$, or equivalently at directional minimum $x^*_j := x + \alpha^*_j e_j$. One can think of sign$([\nabla f(x)]_j)$ as being the true label of $x$, sign$([\nabla f(x)]_j+z)$ as being the observed label, and finding $x_j^*$ as learning the decision boundary (point where labels switch signs). Define regression function
$$\eta(x) := \Pr \Big(\mbox{sign}([\nabla f(x)]_j+z) = +|x \Big)$$
and note that minimizing $f_{x_0,j}$ corresponds to identifying the Bayes threshold classifier as $x_j^*$ because the point at which $\eta(x)=0.5$ or $[\nabla f(x)]_j=0$ is $x_j^*$. Consider a point $x = x^*_j + t e_j$ for $t>0$ with $[\nabla f(x)]_j > 0$ and hence has true label $+$ (a similar argument can be made for $t < 0$). As discussed earlier, $\big| [\nabla f(x)]_j \big| = \bT \Big( \|x-x_j^*\|^{k-1} \Big) = \bT (t^{k-1})$. The probability of seeing label $+$ is the probability that we draw $z$ in $\big(-[\nabla f(x)]_j,\infty \big)$ so that the sign of $[\nabla f(x)]_j+z$ is still positive. Hence, the regression function can be written as
\begin{align*}
\eta(x)  \ &= \  \Pr \Big([\nabla f(x)]_j + z > 0 \Big) \\
 \ &= \  \Pr (z>0) + \Pr \Big(-[\nabla f(x)]_j < z < 0 \Big) \ = \  0.5 + \bT \Big( [\nabla f(x)]_j \Big)
\end{align*}
$$
\implies  \big |\eta(x) - \tfrac{1}{2} \big| \ = \ \bT \Big( [\nabla f(x)]_j \Big) \ = \ \bT \big( t^{k-1} \big) \ = \ \bT \Big( |x-x_j^*|^{k-1} \Big)\label{bz}
$$
Hence, $\eta(x)$ satisfies the TNC with exponent $k$, and an active learning algorithm (next subsection) can be used to obtain a point $\hat{x}_T$ with small point-error and excess risk. Note that function error in convex optimization is bounded above by excess risk of the corresponding active learner using eq (\ref{risk}) because
\begin{align*} \label{ferrorrisk}
f_j(\hat{x}_T) - f_j(x_j^*)  \  &= \ \Bigg| \int\limits^{\hat{x}_T \vee x_j^*}_{\hat{x}_T \wedge x_j^*}  [\nabla f(x)]_j  \mathrm{dx} \Bigg|
\ &= \bT \Bigg( \int\limits^{\hat{x}_T \vee x_j^*}_{\hat{x}_T \wedge x^*_j} |2\eta(x)-1|\mathrm{dx} \Bigg)\\
 \  &=\   \bT \Big(\Risk (\hat{x}_T)\Big)
\end{align*}
Similarly, for stochastic sign oracles (Sec. \ref{sgso}), using $\eta(x) = \Pr \big (\hs_j(x) = + \big) $,
\begin{eqnarray*}
\big| \eta(x) - \tfrac{1}{2} \big| \ = \ \bT \Big([\nabla f(x)]_j\Big) \ = \ \bT \Big (\|x-x^*_j\|^{k-1} \Big)
\end{eqnarray*}

\subsection{A Non-adaptive Active Threshold Learning Algorithm}

One can use a grid-based probabilistic variant of binary search called the BZ algorithm \cite{BZ74} to approximately learn the threshold efficiently in the active setting, in the setting that $\eta(x)$ satisfies the TNC for known $k, \mu, M$ (it is not adaptive to the parameters of the problem - one needs to know these constants beforehand). The analysis of BZ and the proof of the following lemma are discussed in detail in Theorem 1 of \cite{CN09}, Theorem 2 of \cite{CN07} and the Appendix of \cite{RS13}.

\begin{lemma} \label{BZ} 
Given a $1$-dimensional regression function that satisfies the TNC with known parameters $\mu, k$, then after $T$ queries, the BZ algorithm returns a point $\hat{t}$ such that $| \hat{t} - t | = \tT (T^{-\frac{1}{2k - 2}})$ and the excess risk is $\tT (T^{-\frac{k}{2k - 2}})$.
\end{lemma}
Due to the described connection between exponents, one can use BZ to approximately optimize a one dimensional uniformly convex function $f_j$ with known uniform convexity parameters $\lambda,k$. 
Hence, the BZ algorithm can be used to find a point with low function error by searching for a point with low risk. This, when combined with Lemma \ref{BZ}, yields the following important result.
\begin{lemma} \label{perror}
Given a $1$-dimensional $k$-UC and LkSS function $f_j$, a line search to find $\hat{x}_T$ close to $x^*_j$ up to accuracy $|\hat{x}_T - x^*_j| \leq \eta$ in point-error can be performed in $\tT (1/\eta^{2k - 2})$ steps using the BZ algorithm. Alternatively, in $T$ steps we can find $\hat{x}_T$ such that $f(\hat{x}_T) - f(x^*_j) = \tT (T^{-\frac{k}{2k - 2}})$.
\end{lemma}

\section{A 1-D Adaptive Active Threshold Learning Algorithm}

We now describe an algorithm for active learning of one-dimensional thresholds that is adaptive, meaning it can achieve the minimax optimal rate even if the TNC parameters $M,\mu,k$ are unknown. It is quite different from the non-adaptive BZ algorithm in its flavour, though it can be regarded as a robust binary search procedure, and its design and proof are inspired from an optimization procedure from \cite{JN10} that is adaptive to unknown UC parameters $\lambda,k$. 

Even though \cite{JN10} considers a specific optimization algorithm (dual averaging), we observe that their algorithm that adapts to unknown UC parameters can use any optimal convex optimization algorithm as a  subroutine within each epoch. Similarly, our adaptive active learning algorithm is epoch-based and can use any optimal passive learning subroutine in each epoch. We note that \cite{H11} also developed an adaptive algorithm based on disagreement coefficient and VC-dimension arguments, but it is in a pool-based setting where one has access to a large pool of unlabeled data, and is much more complicated.

\subsection{An Optimal Passive Learning Subroutine}

The excess risk of passive learning procedures for 1-d thresholds can be bounded by $\bO (T^{-1/2})$ (e.g. see Alexander's inequality in \cite{DGL96} to avoid $\sqrt{\log T}$ factors from ERM/VC arguments) and can be achieved by ignoring the TNC parameters.

Consider such a passive learning procedure under a uniform distribution of samples (mimicked by active learning by querying the domain uniformly) in a ball\footnote{Define $B(x,R) := [x-R,x+R]$}  $B(x_0,R)$ around an arbitrary point $x_0$ of radius $R$ that is known to contain the true threshold $t$. Then without knowledge of $M,\mu, k$, in $T$ steps we can get a point $\hat{x}_T$ close to the true threshold $t$ such that with probability at least $1-\delta$
$$\Risk (\hat{x}) - \Risk(t) = \int\limits_{\hat{x}_T \vee t}^{\hat{x}_T \wedge t} |2\eta(x) - 1|dx \leq \frac{C_\delta R}{\sqrt T}$$  
for some constant $C_\delta$. Assuming  $\hat{x}_T$ lies inside the TNC region,
$$\mu \int\limits_{\hat{x}_T \vee t}^{\hat{x}_T \wedge t} |x - t|^{k-1} dx \leq \int\limits_{\hat{x}_T \vee t}^{\hat{x}_T \wedge t} |2\eta(x) - 1|dx $$
Hence $\frac{\mu |\hat{x}_T-t|^k}{k} \leq \frac{C_\delta R}{\sqrt T}$. Since $k^{1/k} \leq  2$, w.p. at least $1-\delta$ we get a point-error
\begin{equation}\label{pass}
|\hat{x}_T-t| \leq  2\left[ {\frac{C_\delta R}{\mu \sqrt T}} \right]^{1/k}
\end{equation}
We  assume that $\hat{x}_T$ lies within the TNC region since the interval $|\eta(x) ~-~ \tfrac{1}{2}|~ \leq~ \epsilon_0$ has at least constant width $|x-t| \leq \delta_0 = (\epsilon_0/M)^{1/(k-1)}$, it will only take a constant number of iterations to find a point within it. A formal way to argue this would be to see that if the overall risk goes to zero like $\frac{C_\delta R}{\sqrt T}$, then the point cannot stay outside this constant sized region of width $\delta_0$ where $|\eta(x) -1/2| \leq \epsilon_0$, since it would accumulate a large constant risk of at least $\int\limits_{t}^{t+\delta_0} \mu |x-t|^{k-1} = \frac{\mu \delta_0^k}{k}$. So as long as $T$ is larger than a constant $T_0 := \frac{C_\delta^2 R^2 k^2}{\mu^2 \delta_0^{2k}}$, our bound in eq \ref{pass} holds with high probability (we can even assume we waste a constant number of queries to just get into the TNC region before using this algorithm).

\subsection{Adaptive One-Dimensional Active Threshold Learner} \label{subsec1D}

\begin{algorithm}[ht] \label{adapt}
 \caption{Adaptive Threshold Learner }
 \textbf{Input:} Domain $S$ of diameter $R$, oracle budget $T$, confidence $\delta$\\
 \vspace{1mm}
 \textbf{Black Box:} Any optimal passive learning procedure $P(x,R,N)$ that outputs an estimated threshold in $B(x,R)$ using $N$ queries\\
  \vspace{1mm}
Choose any $x_0 \in S$, $R_1=R, E = \log \sqrt {\frac{2T}{C^2_\td \log  T}}, N = \frac{T}{E}$ 
  \vspace{-2mm}
 \begin{algorithmic}[1]
   \WHILE{$1 \leq e \leq E$}
   \STATE $x_e \leftarrow P(x_{e-1},R_e,N)$
   \STATE $R_{e+1} \leftarrow \frac{R_e}{2}, e \leftarrow e+1$
   \ENDWHILE
 \end{algorithmic}
  \vspace{1mm}
 \textbf{Output:} $x_{E}$ \\
 \vspace{1mm}
 \end{algorithm}

Algorithm \ref{adapt} is a generalized epoch-based binary search, and we repeatedly perform passive learning in a halving search radius. Let the number of epochs be $E := \log \sqrt {\frac{2T}{C_\td^2 \log  T}} \leq \frac{\log T}{2}$ (if$^7$ constant $C_\td^2>2$) and $\td := 2\delta/\log T \leq \delta/E$. Let the time budget per epoch be $N := T/E$ (the same for every epoch) and the search radius in epoch $e \in \{1,...,E\}$ shrink as $R_e := 2^{-e+1} R$.

Let us define the minimizer of the risk within the ball of radius $R_e$ centered around $x_{e-1}$ at epoch $e$ as
$$x^*_e = \arg \min \big\{\Risk (x) : x \in S \cap B(x_{e-1},R_e) \big\} $$
Note that $x^*_e = t$ iff $t \in B(x_{e-1},R_e)$ and will be one end of the interval otherwise.

\begin{theorem} \label{Tadapt}
In the setting of one-dimensional active learning of thresholds, Algorithm 1 adaptively achieves $\Risk (x_{E}) - \Risk (t) = \tO \left( T ^{-\frac{k}{2k-2}} \right)$ with probability at least $1-\delta$ in $T$ queries when the unknown regression function $\eta(x)$ has unknown TNC parameters $\mu,k$.
\end{theorem}

\begin{proof} Since we use an optimal passive learning subroutine at every epoch, we know that after each epoch $e$ we have with probability at least $1 - \td$ \footnote{By VC theory for threshold classifiers or similar arguments in \cite{DGL96}, $C^2_\td \sim \log(1/\td) \sim\log \log T$ since $\td \sim \delta/ \log T$. We treat it as constant for clarity of exposition, but actually lose $\log \log T$ factors like the high probability arguments in \cite{HK11} and \cite{RS13}}\label{loglog}
\begin{equation}\label{perepoch}
\Risk (x_{e}) - \Risk (x^*_e) \leq \frac{C_\td R_e}{\sqrt{T/E}} \leq C_\td R_e \sqrt{\frac{\log T}{2T}} 
\end{equation}
Since $\eta(x)$ satisfies the TNC (and is bounded above by $1$), we have for all $x$
$$\mu |x-t|^{k-1} \leq |\eta(x) - 1/2| \leq 1$$ 
If the set has diameter $R$, one of the endpoints must be at least $R/2$ away from $t$, and hence we get a limitation on the maximum value of $\mu$ as $\mu \leq \frac{1}{(R/2)^{k-1}}$.
Since $k \geq 2$ and $E \geq 2$, and $2^{-E} = C_\td \sqrt{\frac{\log T}{2T}}$, using simple algebra we get 
$$ \mu \leq \frac{ 2^{(k-2)E+2}}{(R/2)^{k-1}} = \frac{4.2^{-E}2^{(k-1)E}2^{(k-1)}}{R^{k-1}} = \frac{4.2^{-E}2^{(k-1)}}{(2^{-E}R)^{k-1}} = \frac{4 C_\td 2^{k-1}}{R_{E+1}^{k-1}} \sqrt{\frac{\log T}{2T}}$$
We prove that we will be appropriately close to $t$ after some epoch $e^*$ by doing case analysis on $\mu$. When the true unknown $\mu$ is sufficiently small, i.e.
\begin{equation}\label{musmall}
\mu \leq  \frac{4C_\td 2^{k-1}}{R_2^{k-1}} \sqrt{\frac{\log T}{2T}}
\end{equation}
then we show that we'll be done after $e^*=1$. Otherwise, we will be done after epoch $2 \leq e^* \leq E$ if the true $\mu$ lies in the range 
\begin{equation}\label{mubig}
\frac{4 C_\td 2^{k-1}}{R_{e^*}^{k-1}} \sqrt{\frac{\log T}{2T}} \leq \mu \leq \frac{4C_\td 2^{k-1}}{R_{e^*+1}^{k-1}} \sqrt{\frac{\log T}{2T}}
\end{equation}
To see why we'll be done, equations (\ref{musmall}) and (\ref{mubig}) imply $R_{e^*+1} \leq 2 \left( \frac{8C_\td^2 \log T}{\mu^2 T} \right)^{\frac{1}{2k-2}}$ after epoch $e^*$ and plugging this into equation (\ref{perepoch}) with $R_{e^*} = 2R_{e^*+1}$, we get
\begin{equation}\label{estar}
\Risk (x_{e^*}) - \Risk (x^*_{e^*}) \leq C_\td R_{e^*} \left( \frac{\log T}{2T} \right)^{\frac1{2}} = \bO \left( \left( \frac{\log T}{T} \right)^{\frac{k}{2k-2}} \right) 
\end{equation}
There are two issues hindering the completion of our proof. The first is that even though $x_1^* = t$ to start off with, it might be the case that $x^*_{e^*}$ is far away from $t$ since we are chopping the radius by half at every epoch. Interestingly, in lemma \ref{before} we will prove that round $e^*$ is the last round up to which $x^*_e = t$. This would imply from eq (\ref{estar}) that
\begin{equation}\label{eqbefore}
\Risk (x_{e^*}) - \Risk (t) = \tO \left( T^{-\frac{k}{2k-2}} \right)
\end{equation}
Secondly we might be concerned that after the round $e^*$, we may move further away from $t$ in later epochs. However, we will show that since the radii are decreasing geometrically by half at every epoch, we cannot really wander too far away from $x_{e^*}$. This will give us a bound (see lemma \ref{after}) like
\begin{equation}\label{eqafter}
\Risk (x_{E}) - \Risk (x_{e^*})  =  \tO \left( T^{-\frac{k}{2k-2}} \right)
\end{equation}
We will essentially prove that the final point $x_{e^*}$ of epoch $e^*$ is sufficiently close to the true optimum $t$, and the final point of the algorithm $x_{E}$  is sufficiently close to $x_{e^*}$. Summing eq (\ref{eqbefore}) and eq (\ref{eqafter}) yields our desired result.

\begin{lemma}\label{before}
For all $e \leq e^*$, conditioned on having $x^*_{e-1}=t$, with probability $1-\td$ we have $x^*_e = t$. In other words, up to epoch $e^*$, the optimal classifier in the domain of each epoch is the true threshold with high probability. \end{lemma}

\begin{proof}
$x_e^* = t$ will hold in epoch $e$ if the distance between the first point $x_{e-1}$ in the epoch $e$ is such that the ball of radius $R_e$ around it actually contains $t$, or mathematically if $| x_{e-1} - t | \leq R_e$. This is trivially satified for $e=1$, and assuming that it is true for epoch $e-1$ we will show show by induction that it holds true for epoch $e \leq e^*$ w.p. $1-\td$. Notice that using equation (\ref{pass}), conditioned on the induction going through in previous rounds ($t$ being within the search radius), after the completion of round $e-1$ we have with probability $1 - \td$
$$|x_{e-1} - t | \leq 2 \left[ {\frac{C_\td R_{e-1}}{\mu \sqrt {T/E}}} \right]^{1/k} $$
If this was upper bounded by $R_e$, then the induction would go through. So what we would really like to show is that $2 \left [\frac{C_\td   R_{e-1}}{\mu \sqrt{T/E}} \right ]^{\frac{1}{k}} \leq R_e$. Since $R_{e-1} = 2R_{e}$, we effectively want to show $\frac{2^k C_\td 2R_e  }{\mu} \sqrt{ \frac{E}{ T}} \leq  R_{e}^k $ or equivalently that for all $e \leq e^*$ we would like to have $\frac{4C_\td 2^{k-1}}{R_{e}^{k-1}} \sqrt{ \frac{ E}{ T}} \leq \mu$. Since $E \leq \frac{\log  T}{2}$, we would be achieving something stronger if we showed
$$ \frac{4C_\td 2^{k-1}}{R_{e}^{k-1}} \sqrt{ \frac{ \log T}{2 T}} \leq \mu$$
which is known to be true for every epoch up to $e^*$ by equation (\ref{mubig}).
\end{proof}

\begin{lemma} \label{after}
For all $e^* < e \leq E$, $\Risk (x_{e}) - \Risk (x_{e^*})  \leq \frac{C_\td R_{e^*}}{\sqrt {T/E}} =  \tO \left( T^{-\frac{k}{2k-2}} \right) $ w.p. $1-\td$, ie after epoch $e^*$, we cannot deviate much from where we ended epoch $e^*$. \end{lemma}

\begin{proof}
For $e > e^*$, we have with probability at least $1-\td$
$$\Risk (x_{e}) - \Risk (x_{e-1})  \leq  \Risk (x_{e}) - \Risk (x^*_e) \leq \frac{C_\td R_e}{\sqrt {T/E}}$$
and hence even for the final epoch $E$, we have with probability $(1 - \td)^{E-e^*}$
$$\Risk (x_{E}) - \Risk (x_{e^*}) = \sum_{e=e^*+1}^E [\Risk (x_{e}) - \Risk (x_{e-1})] \leq \sum_{e=e^*+1}^E \frac{C_\td R_e}{\sqrt {T/E}}$$
Since the radii are halving in size, this is upper bounded (like equation (\ref{estar})) by
$$ \frac{C_\td R_{e^*}}{\sqrt {T/E}} [1/2 + 1/4 + 1/8 +...] \leq \frac{C_\td R_{e^*}}{\sqrt {T/E}} = \tO \left( T^{-\frac{k}{2k-2}} \right)$$
\end{proof}

These lemmas justify the use of equations (\ref{eqbefore}) and (\ref{eqafter}), whose sum yields our desired result. Notice that the overall probability of success is at least $(1 - \td)^E \geq 1 - \delta$, hence concluding the proof of the theorem.

\end{proof}



\section{Randomized Stochastic-Sign Coordinate Descent}

We now describe an algorithm that can do stochastic optimization of $k$-UC and LkSS functions in $d>1$ dimensions when given access to a stochastic sign oracle and a black-box 1-D active learning algorithm, such as our adaptive scheme from the previous section as a subroutine. The procedure is well-known in the literature, but the idea that one only needs noisy gradient signs to perform minimization optimally, and that one can use active learning as a line-search procedure, is novel to the best of our knowledge.

The idea is to simply perform random coordinate-wise descent with approximate line search, where the subroutine for line search is an optimal active threshold learning algorithm that is used to approach the minimum of the function along the chosen direction. Let the gradient at epoch $e$ be called $\nabla_{e-1} = \nabla f(x_{e-1})$, the unit vector direction of descent $d_e$ be a unit coordinate vector chosen randomly from $\{1...d\}$, and our step size from $x_{e-1}$ be $\alpha_e$ (determined by active learning) so that our next point is $x_e := x_{e-1} + \alpha_e d_e$.

Assume, for analysis, that the optimum of $f_e(\alpha) := f(x_{e-1} + \alpha d_e)$ is  
$$\alpha^*_e := \arg \min_\alpha f(x_{e-1} + \alpha d_e) \mbox{ and } x^*_e := x_{e-1} + \alpha_e^* d_e$$
where (due to optimality) the derivative is 
\begin{equation} \label{0deriv}
\nabla f_e(\alpha_e^*) = 0 = \nabla f(x^*_e)^\top d_e
\end{equation}
The line search to find $\alpha_e$ and $x_e$ that approximates the minimum $x^*_e$ can be accomplished by any optimal active learning algorithm algorithm, once we fix the number of time steps per line search.


\subsection{Analysis of Algorithm \ref{rscdd}}
\vspace{-4mm}
\begin{algorithm}[h!] \label{rscdd}
 \caption{Randomized Stochastic-Sign Coordinate Descent}
  \textbf{Input:} set $S$ of diameter $R$, query budget $T$ \\
  \vspace{1.2mm}
  \textbf{Oracle:} stochastic sign oracle $O_f (x,j)$ returning noisy $\mbox{sign}\big([\nabla f(x)]_j \big)$\\
  \vspace{1.2mm}  
  \textbf{BlackBox:} algorithm $LS (x,d,n)$ : line search from $x$, direction $d$, for $n$ steps\\
    \vspace{1.2mm}  
Choose any $x_0 \in S$,  $E = d(\log T)^2$
 \begin{algorithmic}[1]
   \WHILE{$1 \leq e \leq E$}
   \STATE Choose a unit coordinate vector $d_e$ from $\{1...d\}$ uniformly at random
   \STATE $x_e \leftarrow$ $LS(x_{e-1},d_e,T/E)$ using $O_f$
   \STATE $e \leftarrow e+1$
   \ENDWHILE
 \end{algorithmic}
 \textbf{Output:} $x_{E}$\\
     \vspace{1.5mm}  
 \end{algorithm}
 \vspace{-4mm}

Let the number of epochs be $E = d (\log T)^2$, and the number of time steps per epoch is $T/E$. We can do a line search from $x_{e-1}$, to get $x_e$ that approximates $x^*_e$ well in function error in $T/E = \tO(T)$ steps using an active learning subroutine and let the resulting function-error be denoted by $\epsilon' = \tO \Big(T^{-\frac{k}{2k-2}} \Big)$. 

$$f(x_e) \leq f(x_e^*) + \epsilon'$$
Also, LkSS and UC allow us to infer (for $k^* = \frac{k}{k-1}$, i.e. $1/k + 1/k^* = 1$)
$$ f(x_{e-1}) - f(x^*_e)  \ \geq \ \frac{\lambda}{2} \|x_{e-1} - x^*_e\|^k \ \geq \ \frac{\lambda}{2\Lambda^{k^*}} \big| \nabla_{e-1}^\top d_e \big|^{k^*}$$
Eliminating $f(x^*_e)$ from the above equations, subtracting $f(x^*)$ from both sides, denoting $\Delta_e := f(x_e) - f(x^*)$ and taking expectations 
$$
\E[\Delta_{e}] \leq \E[\Delta_{e-1}] - \frac{\lambda}{2\Lambda^{k^*}} \E \Big[ \big| \nabla_{e-1}^\top d_e \big|^{k^*} \Big] + \epsilon'
$$
Since\footnote{$k \geq 2 \implies 1 \leq k^* \leq 2 \implies \| . \|_{k^*} \geq \|.\|_2$} $\E \Big[|\nabla_{e-1}^\top d_e|^{k^*} \big| d_1,...,d_{e-1} \Big] = \frac1{d} \|\nabla_{e-1}\|_{k^*}^{k^*} \geq \frac1{d} \|\nabla_{e-1}\|^{k^*}$ we get
$$
\E[\Delta_{e}] \leq \E[\Delta_{e-1}] - \frac{\lambda}{2d\Lambda^{k^*}} \E \Big[\|\nabla_{e-1}\|^{k^*} \Big] + \epsilon'
$$
By convexity, Cauchy-Schwartz and UC\footnote{$\Delta_{e-1}^k \leq [\nabla_{e-1}^\top(x_{e-1} - x^*)]^k \leq \|\nabla_{e-1}\|^k\|x_{e-1} - x^*\|^k \leq \|\nabla_{e-1}\|^\kappa \frac{2}{\lambda}\Delta_{e-1}$}, $\|\nabla_{e-1}\|^{k^*} \geq \left( \frac{\lambda}{2} \right) ^{1/k-1}\Delta_{e-1}$, we get
$$
\E[\Delta_{e}] \leq \E[\Delta_{e-1}] \left( 1 - \frac1{d} \left( \frac{\lambda}{2\Lambda} \right)^{k^*} \right ) + \epsilon'
$$
Defining\footnote{Since $1 < k^* \leq 2$ and $\Lambda > \lambda/2$, we have $C<1$} $C:= \frac1{d} \left( \frac{\lambda}{2\Lambda} \right)^{k^*} < 1$, we get the recurrence
$$\E[\Delta_{e}] - \frac{\epsilon'}{C} \leq (1-C)\left( \E[\Delta_{e-1}] - \frac{\epsilon'}{C} \right)$$
Since $E = d (\log T)^2$ and $\Delta_0 \leq L\|x_0 - x^*\| \leq LR$, after the last epoch, we have
\begin{align*}
\E[\Delta_E] - \frac{\epsilon'}{C} \ &\leq \ (1-C)^E \left (\Delta_0 - \frac{\epsilon'}{C} \right ) \ \leq \ \exp \big\{-Cd (\log T)^2 \big\} \Delta_0 \ \\ 
&\leq \ LR T^{-Cd \log T} 
\end{align*}
As long as $T > \exp \left\{ (2\Lambda/\lambda)^{k^*} \right\}$, a constant, we have $Cd \log T \geq 1$ and
$$\E[\Delta_E] = \bO (\epsilon') + \mathrm{o}(T^{-1}) = \tO \Big(T^{-\frac{k}{2k-2}} \Big)$$
which is the desired result. Notice that in this section we didn't need to know $\lambda, \Lambda, k$, because we simply run randomized coordinate descent for $E = d (\log T)^2$ epochs with $T/E$ steps per subroutine, and the active learning subroutine was also adaptive to the appropriately calculated TNC parameters. In summary,

\begin{theorem} \label{Tsscd}
Given access to only noisy gradient sign information from a stochastic sign oracle, Randomized Stochastic-Sign Coordinate Descent can minimize UC and LkSS functions at the minimax optimal convergence rate for expected function error of $\tO(T^{-\frac{k}{2k-2}})$ adaptive to all unknown convexity and smoothness parameters. As a special case for $k=2$, strongly convex and strongly smooth functions can be minimized in $\tO(1/T)$ steps.
\end{theorem}

\subsection{Gradient Sign-Preserving Computations}

A practical concern for implementing optimization algorithms is machine precision, the number of decimals to which real numbers are stored. Finite space may limit the accuracy with which every gradient can be stored, and one may ask how much these inaccuracies may affect the final convergence rate - how is the query complexity of optimization affected if the true gradients were rounded to one or two decimal points? If the gradients were randomly rounded (to remain unbiased), then one might guess that we could easily achieve stochastic first-order optimization rates.

However, our results give a surprising answer to that question, as a similar argument reveals that for UC and LkSS functions (with strongly convex and strongly smooth being a special case), our algorithm achieves exponential rates. Since rounding errors do not flip any sign in the gradient, even if the gradient was rounded or decimal points were dropped as much as possible and we were to return only a single bit per coordinate having the true signs, then one can still achieve the exponentially fast convergence rate observed in non-stochastic settings - our algorithm needs only a logarithmic number of epochs, and in each epoch active learning will approach the directional minimum exponentially fast with noiseless gradient signs using a perfect binary search. In fact, our algorithm is the natural generalization for a higher-dimensional binary search, both in the deterministic and stochastic settings.

We can summarize this in the following theorem:

\begin{theorem}
Given access to gradient signs in the presence of sign-preserving noise (such as deterministic or random rounding of gradients, dropping decimal places for lower precision, etc), Randomized Stochastic-Sign Coordinate Descent can minimize UC and LkSS functions exponentially fast, with a function error convergence rate of $\tO(\exp\{-T\})$.
\end{theorem}

\section{Discussion}
While the assumption of smoothness is natural for strongly convex functions, our assumption of LkSS might appear strong in general. It is possible to relax this assumption and require the LkSS exponent to differ from the UC exponent, or to only assume strong smoothness - this still yields consistency for our algorithm, but the rate achieved is worse. \cite{JN10} and \cite{RS13} both have epoch based algorithms that achieve the minimax rates under just Lipschitz assumptions with access to a full-gradient stochastic first order oracle, but it is hard to prove the same rates for a coordinate descent procedure without smoothness assumptions.

Given a target function accuracy $\epsilon$ instead of query budget $T$, a similar randomized coordinate descent procedure to ours achieves the minimax rate with a similar proof, but it is non-adaptive since we presently don't have an adaptive active learning procedure when given $\epsilon$. As of now, we know no adaptive UC optimization procedure when given $\epsilon$.


Recently, \cite{BM11} analysed stochastic gradient descent with averaging, and show that for smooth functions, it is possible for an algorithm to automatically adapt between convexity and strong convexity, and in comparision we show how to adapt to unknown uniform convexity (strong convexity being a special case of $\kappa=2$). It may be possible to combine the ideas from this paper and \cite{BM11} to get a universally adaptive algorithm from convex to all degrees of uniform convexity. It would also be interesting to see if these ideas extend to connections between convex optimization and learning linear threshold functions.

In this paper, we exploit recently discovered theoretical connections 
by providing explicit algorithms that take advantage of them. 
We show how these could lead to cross-fertilization of fields in both directions and hope that this is just the beginning of a flourishing interaction where these insights may lead to many new algorithms if we leverage the theoretical relations in more innovative ways.


\bibliographystyle{agsm}
\bibliography{ALT13}


%
%
%

\end{document}